\pdfoutput=1
\PassOptionsToPackage{dvipsnames}{xcolor} %
\documentclass[11pt]{article}

\usepackage{amsmath}
\usepackage{amssymb}
\usepackage{amsthm} 
\usepackage{algorithm}
\usepackage{algcompatible}

\usepackage[utf8]{inputenc} %
\usepackage[T1]{fontenc}    %
\usepackage{hyperref}       %
\usepackage{url}            %
\usepackage{booktabs}       %
\usepackage{amsfonts}       %
\usepackage{nicefrac}       %
\usepackage{microtype}      %

\usepackage{graphicx}
\usepackage{multirow}
\usepackage[noend]{algpseudocode}
\usepackage{setspace}

\newcommand{\bc}{\mathbf{c}}

\newcommand{\bs}{\mathbf{s}}

\newcommand{\bx}{\mathbf{x}}
\newcommand{\by}{\mathbf{y}}

\newcommand{\bz}{\mathbf{z}}

\newcommand{\CD}{\mathcal{D}}

\newcommand{\CO}{\mathcal{O}}

\newcommand{\CX}{\mathcal{X}}
\newcommand{\CZ}{\mathcal{Z}}

\providecommand{\norm}[1]{\left\lVert#1\right\rVert}

\newcommand{\bbR}{\mathbb{R}}

\DeclareMathOperator{\softmax}{softmax}
\DeclareMathOperator{\clip}{clip}
\DeclareMathOperator{\cluster}{cluster}

\DeclareMathOperator{\logits}{logits}

\DeclareMathOperator{\median}{median}
\DeclareMathOperator{\aggregate}{aggregate}
\DeclareMathOperator{\batch}{batch}
\DeclareMathOperator{\embed}{embed}

\DeclareMathOperator{\leftmedian}{left-median}
\DeclareMathOperator{\rightmedian}{right-median}

\DeclareMathOperator{\leftop}{left}
\DeclareMathOperator{\rightop}{right}

\newcommand{\baz}{\bar{z}}

\newcommand{\bbaz}{\bar{\bz}}

\newcommand{\eps}{\varepsilon}

\newtheorem{thm}{Theorem}

\newtheorem{lem}{Lemma}
\newtheorem{defn}{Definition}



\usepackage{subfig}
\usepackage{float}
\usepackage{adjustbox} %
\usepackage{tablefootnote}

\usepackage[dvipsnames,svgnames]{xcolor}

\usepackage[most]{tcolorbox}
\usepackage{xcolor}
\usepackage{verbatim}
\usepackage{upquote}

\usepackage{fullpage}
\usepackage{natbib}
\usepackage{mathtools}
\usepackage{microtype}
\usepackage{longtable}
\usepackage{booktabs}
\usepackage{color}
\usepackage{verbatim}
\usepackage{enumitem}
\usepackage[auth-lg]{authblk}
\definecolor{lightgray}{gray}{0.9}

\title{Clustering and Median Aggregation Improve \\ Differentially Private Inference\thanks{Authors ordered alphabetically.  Author contributions are listed at the end.}}

\author[1]{Kareem Amin}
\author[2]{Salman Avestimehr}
\author[2]{Sara Babakniya}
\author[1]{Alex Bie}
\author[1]{\authorcr Weiwei Kong}
\author[1]{Natalia Ponomareva}
\author[1]{Umar Syed}
\affil[1]{Google Research, New York}
\affil[2]{University of Southern California, Los Angeles}
\date{}
\begin{document}

\maketitle

\begin{abstract}

Differentially private (DP) language model inference is an approach for generating private synthetic text. A sensitive input example is used to prompt an off-the-shelf large language model (LLM) to produce a similar example. Multiple examples can be aggregated together to formally satisfy the DP guarantee. 

Prior work creates inference batches by sampling sensitive inputs uniformly at random. We show that uniform sampling degrades the quality of privately generated text, especially when the sensitive examples concern heterogeneous topics. 

We remedy this problem by clustering the input data before selecting inference batches. Next, we observe that clustering also leads to more similar next-token predictions across inferences. We use this insight to introduce a new algorithm that aggregates next token statistics by privately computing medians instead of averages. This approach leverages the fact that the median has decreased local sensitivity when next token predictions are similar, allowing us to state a data-dependent and ex-post DP guarantee about the privacy properties of this algorithm. Finally, we demonstrate improvements in terms of representativeness metrics (e.g., MAUVE) as well as downstream task performance. We show that our method produces high-quality synthetic data, at significantly lower privacy cost, than a previous state-of-the-art method.
\end{abstract}

\section{Introduction}

One of the many applications for powerful generative AI models is the creation of synthetic data. A natural approach is to prompt a large language model (LLM) with a rewriting task and a representative example, asking it produce synthetic analogs that resemble the example. This approach is not privacy-preserving if the seed example contains sensitive information that could theoretically pass through into the synthetic outputs. 

This limitation is especially problematic if preserving the privacy of the source data was the reason to generate synthetic data in the first place. Consider a data steward who has access to a collection of medical records. They must preserve the privacy of the patients who provided the records. At the same time, they would like to make a privacy-preserving synthetic version of the data public to improve machine learning methods for making diagnoses.

The literature on \emph{differentially private (DP) inference} \citep{dwork2018privacy,papernot2017semi,papernot2018scalable,wu2024prompt,ginart2022submix,Majmudar2022,duan2023flocks,flemings2024adaptively,flemings2024differentially} provides a means to generate synthetic data \citep{hong2023dp,tang2024privacy,amin2024private,gao2025dataadaptive} by prompting a pre-trained model, while ensuring formal privacy guarantees. At a high-level, DP inference methods work by prompting an off-the-shelf LLM for multiple responses, with each one seeded by a sensitive example belonging to a different user. These responses are then aggregated in some way that satisfies DP. Through this procedure, an aggregated response does not represent any single seed example, but is a noisy amalgamation of all the seed examples. 

\paragraph{Our contributions.}

In this work, we study the quality of synthetic data produced in this manner. In particular, we are interested in how the \emph{heterogeneity} of the seed batch affects the \emph{representativeness} of synthetic data. 

DP requires that an adversary cannot detect any single seed example by observing the aggregated response. Thus, if data is highly heterogeneous, this presents a problem; by design, the aggregated response will not be representative of any seed example. In contrast, if all seed examples are highly self-similar, all responses will be similar, and the aggregated response can be representative of all the seed examples without violating the DP guarantee. 

Armed with this observation we note that all state-of-the-art DP inference methods~\citep{tang2024privacy,amin2024private, gao2025dataadaptive} batch seed examples uniformly at random, tending to generate heterogenous batches. We depart from this approach, demonstrating a practical technique for pre-clustering data while still preserving privacy. We use this clustering to assign similar examples to the same batch, creating more homogenous inputs for the DP inference algorithm. 

Next, we propose a new algorithm for DP inference designed to make better aggregations when the LLM's predictions are aligned. We modify the algorithm of \citet{amin2024private}, which aggregates LLM responses on a per-token basis by averaging token logit scores across inferences. Our algorithm replaces this average with a median. It is well-known in the DP literature that a median operation has local sensitivity that depends on how well-concentrated its inputs are. We prove a formal guarantee that holds in the data-dependent~\citep{papernot2017semi,papernot2018scalable} and ex-post~\citep{ligett2017accuracy} differential privacy setting.

\begin{figure*}[t!] 
    \centering 
    \subfloat[DP Inference with Homogenous Data]{\includegraphics[width=0.62\textwidth]{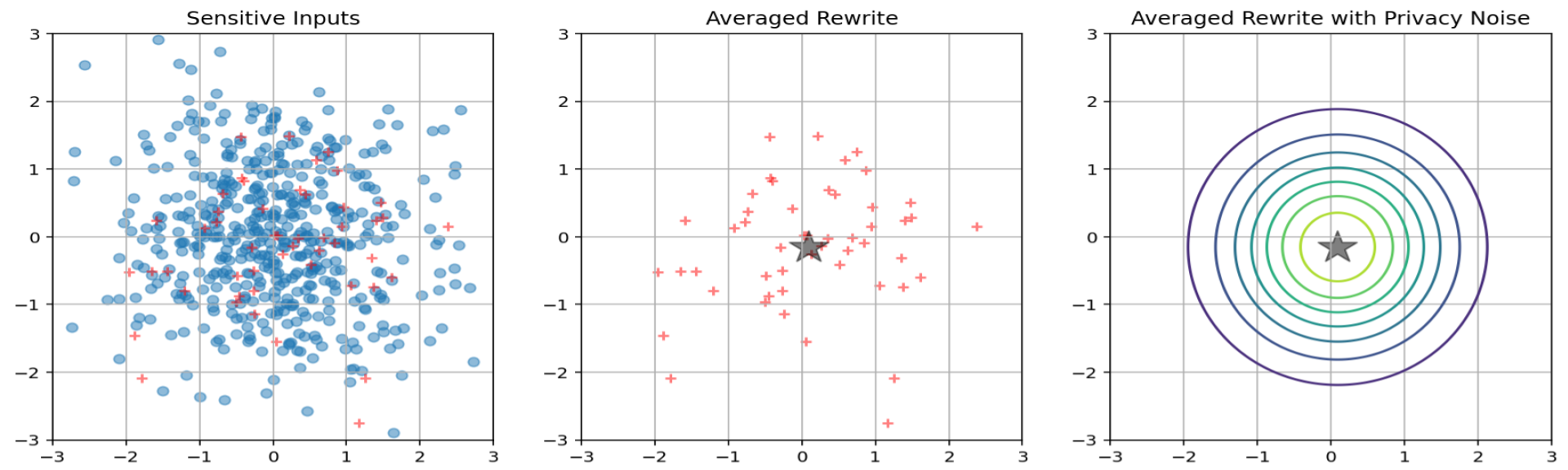}}
    \hfill 
    \subfloat[DP Inference with Heterogenous Data]{\includegraphics[width=0.62\textwidth]{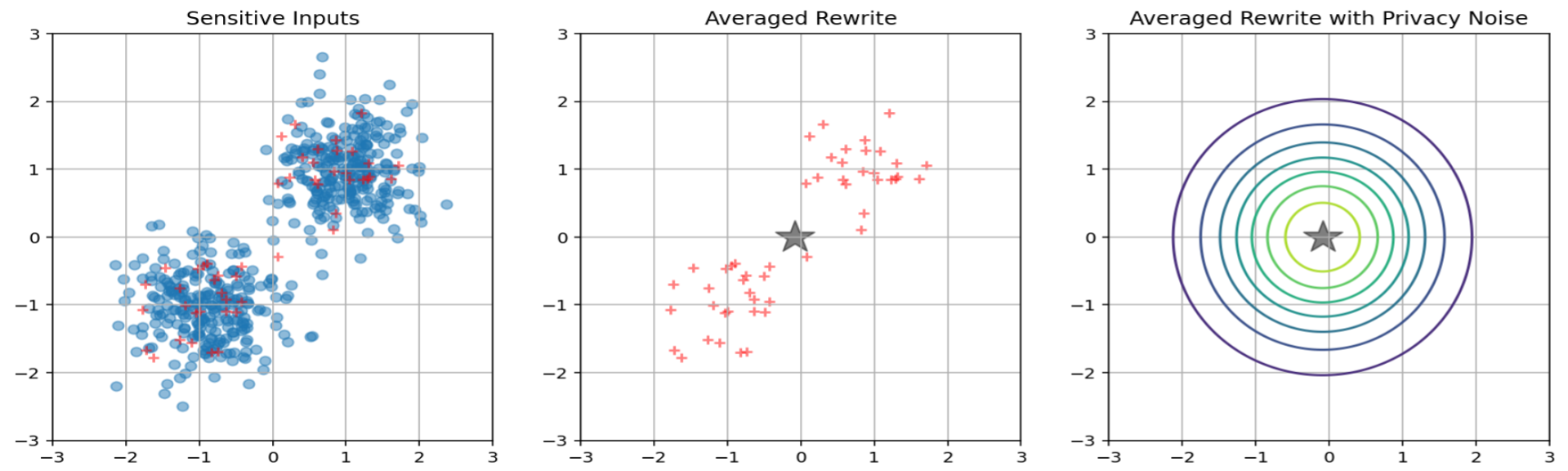}}
    \caption{A stylized depiction of synthetic data generation based on DP inference. \emph{Left (a and b):} An input corpus sits in some embedding space, and is sampled uniformly at random (points in {\color{red} red}) to form a batch. \emph{Middle (a and b):} A depiction of an average rewrite for the batch. \emph{Right (a and b):} A noise distribution centered around the average rewrite. In (b), the distinct semantic clusters found in the input dataset collapse.
    }
    \vspace{-5pt}
    \label{fig:intuition} 
\end{figure*} %

Representativeness metrics like MAUVE~\citep{pillutla2021mauve} have not been previously evaluated for data generated via DP inference. Only recently, \citet{amin2024private} demonstrated generating enough data to begin measuring similarity at a dataset level. However, they report only accuracy measures on downstream tasks. Indeed, we show that their method fails to produce representative data as measured by MAUVE. 

We conduct experiments on a variety of datasets and report improvement on two metrics: MAUVE scores computed on the raw synthetic dataset and accuracy of a BERT model trained on synthetic data. Finally, we incorporate a number of other improvements to further the state-of-the-art MAUVE scores, demonstrating the effect of other design decisions such as prompts, pre-trained vs. instruction-tuned generators, and varying the number of examples used when prompting.

\textbf{To summarize:} In this work, we improve the representativeness of data produced by state-of-the-art methods for private inference. We hypothesize that poor representativeness, as measured by MAUVE score, is due to misalignment of LLM predictions during inference. To correct this, we cluster inputs to achieve better alignment, and therefore more representative data. Our contributions include a study of various clustering techniques, where we identify an effective way to privately cluster data. Finally, we take further advantage of better aligned inputs by aggregating logit scores with a private median, allowing us to state data-dependent, ex-post, differential privacy guarantees at a reduced privacy cost. We perform extensive experiments, isolating each of these contributions on several datasets.  

\section{Limitations of Uniform Batching}
\label{sec:limitations}

As previously discussed, DP inference takes many text inputs (a batch) and attempts to produce a single output representing an \emph{aggregated rewrite} for all of the records in the batch. At this level of abstraction, we can recognize a problem. Batches are ordinarily drawn uniformly at random from the input corpus. Therefore, the aggregated rewrite targeted by the these algorithms will collapse any of the variation within the corpus. Consider the visualization in Figure~\ref{fig:intuition}, where we think of text data sitting in some embedding space, and the average response as a simple average within the embedding space. The distinct semantic clusters in in Figure~\ref{fig:intuition}(b) collapse due to the averaging procedure.

\subsection{Empirical demonstration} 

We can demonstrate this claim by evaluating the performance of a private inference method on a metric that captures the representativeness of the data generated. While the algorithm of~\citet{amin2024private} is known to produce data that performs well on down-stream classification tasks, these results do not tell us whether the synthetic data distribution represents the initial corpus. For that, we use MAUVE \citep{pillutla2021mauve}, a generic comparison measure between text corpora.

\begin{table}[h!]
\small
    \centering
    \begin{tabular}{l l r r}
        \toprule
        \multicolumn{1}{l}{\textbf{Privacy $\boldsymbol{\epsilon}$}} & \multicolumn{1}{l}{\textbf{Method}}                          & \multicolumn{1}{l}{\textbf{MAUVE}}           & \multicolumn{1}{l}{\textbf{Accuracy}}          \\
        \midrule
        \multirow{4}{*}{$\epsilon = \infty$}
        & Real data  &  $.872_{.018}$ & $.965_{.001}$ \\ 
        \cmidrule{2-4}
        & Baseline \citep{amin2024private} &   $.130_{.009}$  & $.892_{.015}$   \\
        \cmidrule{2-4}
        & Baseline++ (w/ pretrained model \& prompt)      &  $.460_{.050}$ & $.898_{.022}$ \\
        & \hspace{0.5em} + \emph{non-private clustering} & $.650_{.021}$ & $.912_{.020}$ \\
        \bottomrule
    \end{tabular}
    \vspace{10pt}
    \caption{\textbf{Clustering improves DP inference results at $\boldsymbol{\epsilon = \infty}$ on Yelp100k using Gemma 2 2B.} We report mean and std of \emph{MAUVE} scores against real data (5 seeds), as well \emph{Accuracy} of a BERT model trained on synthetic data and evaluated on real data (3 seeds). While \cite{amin2024private} enables generation of large synthetic corpora with DP inference, \emph{quantity} begets the question of \emph{representativeness}. \textbf{\emph{Baseline} demonstrates the limits of existing approaches even when privacy is not a concern.} First, we show direct improvements via switching to the pretrained checkpoint and incorporating multiple examples into the prompt (\emph{Baseline++}). On top of these improvements, \textbf{cluster-informed batching leads to improvements in representativeness}. Here, clustering is performed non-privately by running $k$-means on private data, %
    with $K=500$. The remainder of the paper seeks to realize and improve upon these gains under privacy constraints. 
}\label{tab:clustering-demo}
\end{table}

In Table~\ref{tab:clustering-demo} we see that DP inference (c.f. \emph{Baseline}) does not produce representative datasets, even when pushing the methods to their limit by selecting parameters that offer no formal privacy guarantee (the $\epsilon = \infty$ regime). We begin our investigation from an improved baseline (\emph{Baseline++}) obtained by (1) switching from Gemma 2 2B IT to the PT checkpoint (and necessarily changing the prompt); and (2) adding more in-context examples; full details are in Section \ref{subsec:experiment-setup}).

Conceptually, we can remedy the problem of heterogenous batches by first clustering the data, and then constructing batches by uniformly sampling inputs from \emph{within each cluster}. For example, one could alternate between selecting batches from each of the 2 clusters in Figure~\ref{fig:intuition}. In Table~\ref{tab:clustering-demo}, we report the MAUVE score of an algorithm (\emph{non-private clustering}) that aims to do just that. The algorithm computes embeddings of the input corpus with Gecko \citep{lee2024gecko}, and clusters them into 500 clusters using $k$-means. Batches are then constructed by first assigning inputs to clusters and feeding inputs with the same cluster assignment to the algorithm of~\citet{amin2024private}.%

While this procedure significantly improves MAUVE scores, it does not satisfy the DP guarantee. In the remainder of the paper, we describe: (1) how to implement this idea in a privacy-preserving manner; and (2) a new DP inference algorithm that takes advantage of pre-clustered data.

\section{Preliminaries and notation}

Let $\CX$ be the token vocabulary, \emph{i.e.}, the set of all possible tokens. A \emph{token sequence} is an element of $\CX^*$, and a \emph{logit vector} is an element of $\bbR^\CX$ (one logit per token in the vocabulary). For brevity we define $\CZ \equiv \bbR^\CX$ to be the set of all logit vectors. If $\bz \in \CZ$ then $z_x \in \bbR$ denotes the component of $\bz$ corresponding to token $x \in \CX$. If $\bx_1$ and $\bx_2$ are token sequences then we write $\bx_1\bx_2 \in \CX^*$ to denote their concatenation. A \emph{large language model (LLM)} is defined by a function $\logits : \CX^* \rightarrow \CZ$ that maps each token sequence to a logit vector. A \emph{dataset} $D \subseteq \CX^*$ is a subset of token sequences. A pair of sets are \emph{neighbors} if their symmetric difference has size 1, \emph{i.e.}, one set can be formed from the other by adding or subtracting a single element.

\section{Improved algorithm for DP inference}

Algorithm \ref{alg:main} is our method for generating private synthetic text. Given a dataset of sensitive seed texts, the algorithm first partitions the seeds into $m$ batches. For each batch, the algorithm generates a single synthetic example consisting of $n$ tokens. Each synthetic example $\bx$ is generated one token at a time, by first initializing $\bx$ to be the empty token sequence and then repeatedly executing the following procedure: (1) generate $\logits(\bs\bx)$ for each seed $\bs$ in the batch, and aggregate the logit vectors into a single logit vector $\bbaz$; (2) draw token $x$ from $\softmax(\bbaz/\tau)$, the distribution that assigns probability proportional to $\exp(\baz_y/\tau)$ to each token $y$; (3) append $x$ to $\bx$.

\begin{center}
\begin{algorithm}[ht]
\small
\setstretch{1.15}
\caption{\label{alg:main} Generate private synthetic examples}
\begin{algorithmic}[1]
\Statex {\bf Given:} $\logits : \CX^* \rightarrow \CZ$, temperature $\tau > 0$, maximum token sequence length $n > 0$, {\color{Blue} $\batch : \CX^* \rightarrow [m]$}, {\color{Blue} $\aggregate : 2^{\CZ} \rightarrow \CZ$}. 
\Statex {\bf Input:} Dataset of sensitive seeds $D \subseteq \CX^*$.
\Statex {\bf Output:} Dataset of synthetic examples $X \subseteq \CX^*$.
\State $X \gets \emptyset$
\For{each $i = 1, \ldots, m$}
\State $S_i = \{\bs \in D : {\color{Blue} \batch(\bs)} = i\}$.
\State $\bx_{i,0} \gets \textrm{Empty token sequence}$
\For{$t = 1, \ldots, n$}
\State $Z_{i,t} \gets \{\logits(\bs\bx_{i,t-1}) : \bs \in S\}$
\State $\bbaz_{i,t} \gets {\color{Blue} \aggregate(Z_{i,t})}$
\State $x_{i,t} \sim \softmax(\bbaz_{i,t} / \tau)$
\State Append $x_{i,t}$ to $\bx_{i,t-1}$ to form $\bx_{i,t}$
\EndFor
\State $X \gets X \cup \{\bx_{i,n}\}$
\EndFor
\State \textbf{return} $X$
\end{algorithmic}
\end{algorithm}
\end{center}

Algorithm \ref{alg:main} is a generalization of conventional non-private LLM inference, as well as the DP inference method of~\citet{amin2024private}. The differences between the methods are in their implementations of the $\batch()$ and $\aggregate()$ subroutines, which are marked in {\color{Blue} blue} in Algorithm \ref{alg:main}. In conventional inference, $\batch()$ assigns each seed to its own unique batch, and $\aggregate()$ has no effect. In the method from~\citet{amin2024private}, $\batch()$ assigns each seed to one of $m$ batches uniformly at random (typically $m$ is much smaller than the number of seeds), and $\aggregate()$ is defined
\[
\aggregate(Z) = \frac1{|Z|} \sum_{\bz \in Z} \clip_c(\bz)
\]
where $\clip_c(\bz)_i = \max\{-c, z_i - \max_j \{z_j\} + c\}$. In other words, $\aggregate()$ shifts and clips each logit value so that it lies in the interval $[-c, c]$, and then averages the clipped logit vectors together.
Clipping is key to proving a privacy guarantee, which is based on the observation that the token sampling procedure is equivalent to the exponential mechanism \citep{mcsherry2007mechanism}.

\subsection{Batching by clustering}
\label{subsec:clusters}

\begin{table}[t!]
\small
    \centering
    \begin{tabular}{ccccc}
        \toprule
        \multicolumn{1}{l}{\textbf{Clustering Method}} & \multicolumn{1}{l}{\textbf{$\#~$Clusters}} & 
        \multicolumn{1}{l}{\textbf{$\#~$Clusters (size $\ge 100$)}}   &
        \multicolumn{1}{l}{\textbf{Privacy $\eps$}} &
        \multicolumn{1}{l}{\textbf{V-measure}}\\
        
        \midrule
        \textbf{$k$-means} & $500$ & $464$ & $\infty$ & $1$ \\
        \midrule
        \textbf{DP Clustering} & \multirow{2}{*}{$9$} & \multirow{2}{*}{$6$} & \multirow{2}{*}{$0.1$} & \multirow{2}{*}{$0.1$}\\
        \textbf{\citep{chang2021practicaldpclustering}}\\
        \midrule
        \textbf{DP Clustering} & \multirow{2}{*}{$2$} & \multirow{2}{*}{$2$} & \multirow{2}{*}{$10$} & \multirow{2}{*}{$0.01$}\\
        \textbf{\citep{liebenow2024dpmclusteringsensitivedata}} \\
        \midrule
        \textbf{Public Centers} & $438$ & $140$ & $0$ & $0.59$\\
        \midrule
        \textbf{Public Centers} & \multirow{2}{*}{$100$} & \multirow{2}{*}{$100$} & \multirow{2}{*}{$0.1$} & \multirow{2}{*}{$0.58$}\\
        \textbf{ with Rebalancing} \\
        \bottomrule
    \end{tabular}
    \vspace{10pt}
    \caption{Comparing different clustering methods. $\#~$Clusters is the number of non-singleton clusters, $\#~$Clusters ($\ge 100$ Samples) is the number of clusters with at least 100 samples. V-measure \citep{rosenberg2007v} is a metric to compare the quality of each clustering. Higher V-measure shows more similarity to ground truth ($k$-means with $k=500$). The parameter $\eps$ indicates the privacy cost, with higher values indicating higher privacy cost (see Section \ref{sec:analysis}).} 
    \label{tab:clustering-comparison}
\end{table}

Instead of assigning seeds to batches randomly, in this paper we explore the impact of grouping similar seeds together. We consider implementations of $\batch()$ in Algorithm \ref{alg:main} that have the form
\begin{equation}
\batch(\bs) = (\cluster(\bs), r) \label{eq:batch}
\end{equation}
where $\cluster()$ is a cluster assignment function, and $r$ is chosen uniformly at random from $[b]$. In other words, the seed is first assigned to a cluster, and then within that cluster it is randomly assigned to one of $b$ batches.

The cluster assignment function is implemented using a sentence embedding model $\embed()$, which maps a given input text into a fixed-dimensional embedding space $\bbR^d$. A well-trained sentence embedding model will place similar texts closer together in this space. Given cluster centers $\bc_1, \ldots, \bc_k \in \bbR^d$ the cluster assignment function is defined
\[
\cluster(\bs) = \arg \min_{i \in [k]} \norm{\embed(\bs) - \bc_i}_2.
\]
Thus the batching procedure is fully specified by describing how the cluster centers are selected. We consider the following three methods, each of which has different privacy implications.

\begin{figure}[t!]
    \centering
    {\includegraphics[width=0.99\textwidth]{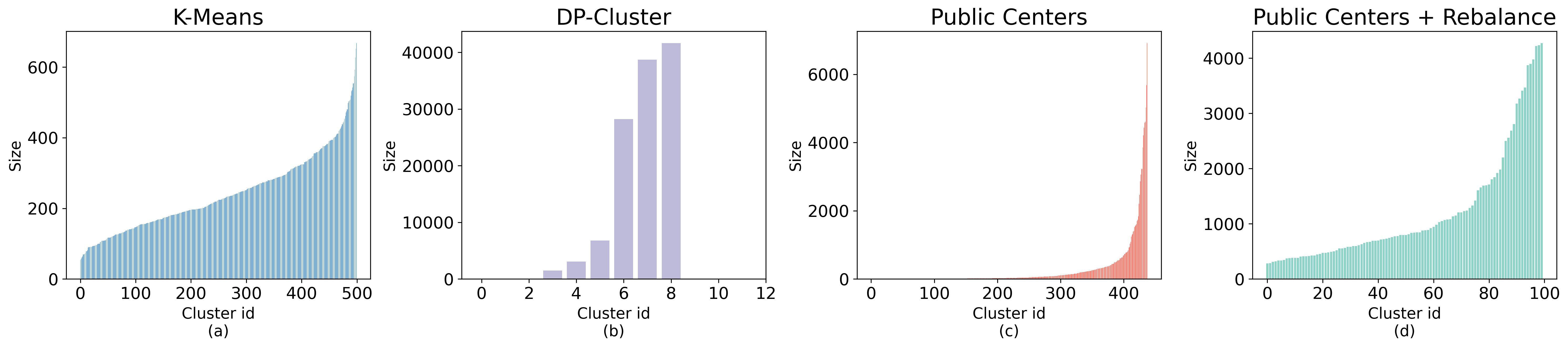}} 
  \caption{Cluster sizes of different clustering methods for AGNews dataset. a) $k$-means ($k=500$) which is not private but gives the most balanced clusters. b) DP-clustering \citep{chang2021practicaldpclustering} ($\eps=0.1$) which is private but most of the data is assigned to a few clusters. c) Clustering with public dataset (DBPedia, $k=500$), which is private and has more valid clusters but still many clusters have only a few examples. d) Clustering with public centers and rebalancing (DBPedia, $k=500$ and rebalancing to $100$ clusters, $\eps=0.1)$. This method does not have any small clusters.} 
  \label{fig:clustering}
\end{figure}

\paragraph{Differentially private centers.} We apply state-of-the-art DP clustering methods \citep{chang2021practicaldpclustering,liebenow2024dpmclusteringsensitivedata} to the seed embeddings to discover the cluster centers. We observe that these methods often fail to find good cluster centers. Most DP clustering algorithms are designed for low-dimensional data, since the amount of privacy-preserving noise injected by the algorithms increases with the dimension, whereas sentence embeddings are typically high-dimensional. 
Figure~\ref{fig:clustering}(b) shows one of the problems with DP clustering \citep{chang2021practicaldpclustering}. Even though the number of target cluster centers $k$ is set to $500$, the algorithm only finds $<10$ non-singleton centers, leading to highly imbalanced clusters.

\paragraph{Public centers.} Given the limitations of privately clustering the seeds, we leverage high-quality public datasets instead. These datasets contain diverse examples, making them useful for clustering. We applied $k$-means clustering to the public data, and used the resulting centers to assign cluster labels to the seeds. Because selecting the cluster centers does not require examining any sensitive data, it does not incur any privacy cost. However, this approach introduces a new challenge: if the public data distribution differs significantly from the sensitive data distribution, the resulting clusters can become highly imbalanced, some with very few examples, and others disproportionately large. Very small clusters are often unusable, while large clusters may still contain heterogeneous data, reducing overall utility. Figure \ref{fig:clustering}(c) illustrates the imbalance problem for public cluster centers.  

\paragraph{Public centers with rebalancing.} To address the cluster imbalance issue, we introduce two additional steps into the process of selecting cluster centers from public data. After obtaining public cluster centers, we compute a noisy count of the seeds assigned to each cluster; this step incurs only a small privacy cost ($\varepsilon\approx0.1$). We then select the cluster centers with the $k'$ highest counts and re-assign the seeds using only these top-$k'$ centers. For example, if we had $k=500$ centers originally, we may reduce them to $k' = 100$ centers after rebalancing. This refinement ensures more balanced clusters while preserving quality, as shown in Figure \ref{fig:clustering}(d), improving both efficiency and utility. Indeed, \citet{tan2025synthesizing} uses a similar technique to improve DP finetuning.

\subsection{Median aggregation}

In addition to improving synthentic data quality, making the batches more homogeneous allows us to modify Algorithm \ref{alg:main} to yield a tighter privacy analysis. Instead of aggregating the clipped logit vectors in a batch by taking their average, we compute their component-wise median:
\begin{equation}
\aggregate(Z) = \median(\{\clip_c(\bz) : \bz \in Z\}), \label{eq:aggregate}
\end{equation}
where $\median(Z)$ is the vector in which each component is equal to the median value of the corresponding components of the vectors in $Z$.

Previous analyses of differentially private inference algorithms for synthetic data generation, such as in \citet{amin2024private} and \citet{tang2024privacy}, were based on the \emph{global} sensitivity of the mean, \emph{i.e.}, on how much the mean of \emph{any} set of logit vectors $Z$ can change when a vector is added or removed from $Z$. Our privacy analysis (in Section \ref{sec:analysis}) is based on the \emph{local} sensitivity of the median, \emph{i.e.}, on how much the component-wise medians of the \emph{actual} set of logit vectors $Z$ can change when a vector is added or removed from $Z$. To see why see the latter sensitivity can be much smaller than the former, note that if $Z$ contains at least 3 identical vectors, then the local sensitivity of $\median(Z)$ is zero. When a batch of seeds texts are all similar to each other, then the logit vectors of their next-token distributions will also be similar, and we exploit this similarity to prove a stronger privacy guarantee than previous work (albeit one that is both data-dependent and output-dependent; see next section).

\paragraph{Significance of clipping.} Given a logit vector $\bz$, $\clip_c$ shifts each component by the same quantity so that $\max_j z_j$ becomes $c$. Any score below $-c$ is then clipped to $-c$. This operation bounds the range of possible values in $\bz$ to $[-c,c]$ and does not complicate the privacy analysis since it operates locally on each $\bz$. While medians are invariant to certain types of shifts, it is important to note that $\clip_c$ applies a different shift to each $\bz$, while $\median$ aggregates across vectors. As a result, $\clip_c$ plays an important role in lowering the local sensitivity of the data. As a very simple example, consider a situation where all inferences are exactly aligned on the next-token distribution. Even if this alignment occurs, there is no reason that the \emph{logit scores} would be aligned since the $\softmax$ operator is scale invariant. In contrast, the clipping operator forces alignment, while preserving the head of the distribution. In this example, the logit score for the most-likely next token will be mapped to $c$, driving the local sensitivity for this token down to zero. 

\section{Privacy analysis}
\label{sec:analysis}

We prove a differential privacy guarantee for Algorithm \ref{alg:main}, which provides an upper bound on the change in the output distribution of the algorithm for any small change to its input. As is standard in differential privacy analyses, the upper bound is expressed in terms of a privacy parameter $\eps$. In contrast to the original definition of differential privacy \citep{dwork2006calibrating}, and following more recent work \citep{papernot2017semi,ligett2017accuracy,papernot2018scalable,chowdhury2020data,ginart2022submix,duan2023flocks,flemings2024adaptively}, we allow $\eps$ to depend on both the input and output of the algorithm, which leads to a \emph{data-dependent ex-post} guarantee.

\begin{defn}[Data-dependent ex-post differential privacy] \label{defn:dp} Let $A : \CD \rightarrow \CO$ be an algorithm. Let $\eps : \CD \times \CO \rightarrow \bbR^{\ge 0}$.  Algorithm $A$ satisfies \emph{$\eps$-data-dependent ex-post differential privacy} if for all neighboring datasets $D, D' \in \CD$ and $X \in \CO$
\[
\exp(-\eps(D, X)) \cdot \Pr[A(D') = X] \le \Pr[A(D) = X] \le \exp(\eps(D, X)) \cdot \Pr[A(D') = X]
\]
\end{defn}

Definition \ref{defn:dp} reduces to the original definition of differential privacy if we require the $\eps$ function to be a constant function. In that special case, the privacy guarantee is a property of the algorithm itself, and holds for worst-case input and output. However, Definition \ref{defn:dp} offers the possibility of a privacy guarantee that is more refined than the worst case, reflecting the fact that certain inputs and outputs have lower privacy risk than others. In any case, the semantic meaning of the privacy guarantee is the same as in the original definition of differential privacy: it quantifies the ability of an adversary to distinguish a small change in the input to the algorithm by only examining its output.

\begin{thm}[Privacy guarantee] \label{thm:main} Algorithm \ref{alg:main} with $\batch()$ and $\aggregate()$ set as in Eq.~\eqref{eq:batch} and Eq.~\eqref{eq:aggregate} satisfies $\eps$-data-dependent ex-post differential privacy, where
\[
\eps(D, X) = \max_{i \in [m]} \sum_{t=1}^{n} \gamma(Z_{i,t}, x_{i, t}),
\]
and the \emph{per-token privacy cost} function $\gamma$ is defined in Appendix \ref{sec:proofs}.
\end{thm}
A formal definition of the function $\gamma$ in Theorem \ref{thm:main} is given in Appendix \ref{sec:proofs}, and here we offer some intuition for why it quantifies the privacy cost of Algorithm \ref{alg:main}. For a set of logit vectors $Z$, let $Z^{(x)}$ be the values in the $x^{\textrm{th}}$ component of each of the vectors. These are the logit scores corresponding to token $x$. If we sort the logit scores in $Z^{(x)}$ in ascending order, then the median is the middle value, and the values that are adjacent to the median define what we call the \emph{median gap}. When the adjacent values are far from the median, the median gap is large, and otherwise it is small. The size of the median gap determines the local sensitivity of the median, since adding or removing a value from $Z^{(x)}$ can cause the median to shift to one of the adjacent values. The function $\gamma(Z, x)$ is an increasing function of the median gap of $Z^{(x)}$, and so higher median gaps lead to higher privacy cost.

\section{Experiments}

\subsection{Experiment setup}\label{subsec:experiment-setup}
\paragraph{Models.} We use Gemma 2 2B models \citep{gemmateam2024gemma2} as the generator for all experiments. We use both the pre-trained (PT) and instruction-tuned (IT) variants of the models. Note the variants necessitate using different prompts, which we give in Appendix \ref{sec:prompts}. All tasks use the same generic prompt template.

\paragraph{Datasets.} Our algorithms utilize a \emph{public dataset} in addition to the target \emph{private dataset} we aim to synthesize. For private datasets in our experiments: we use \emph{AGNews}, \emph{Yelp}, and \emph{NYT Topics}; all of which are equipped with a multi-class classification task. We use \emph{DBPedia} as our sole public dataset for computing public clusters used in all experiments. We chose this dataset since it is based on Wikipedia, which (a) contains a wide variety of topics and therefore is a good candidate for universal clusters; and (b) reflects the kind of public data permissible for use in real deployments. For further details on all datasets used, see Appendix \ref{sec:additional-experiment-details}.

\renewcommand{\arraystretch}{0.87}
\begin{table}[t!]
    \centering
    \scalebox{0.95}{
        \small
		\begin{tabular}{l l c c r r}
		\toprule
		    \textbf{Dataset} & 
            \textbf{Method} & 
            \textbf{Privacy $\varepsilon$} & 
            \textbf{Clusters} & 
            \textbf{MAUVE} & 
            \textbf{Accuracy} \\
			\midrule
			\multirow{11}{*}{AGNews} & Real data & $\infty$ & - & $.872_{.032}$ & $.938_{.001}$ \\
			\cmidrule{2-6}
            & Mean Baseline \citep{amin2024private} & 10 & 4 &  $.156_{.024}$ &   $.704_{.009}$\\
            \cmidrule{2-6}
            & Mean Baseline++ & 10 & 4 & $.633_{.022}$ & $.851_{.015}$ \\
            & Mean Clustered & $9.90+0.1$ & 60 & $.692_{.029}$ & $.855_{.012}$ \\
            & Median Clustered & $2.40+0.1$ & 60 & $.713_{.027}$ & $.868_{.002}$ \\
			\cmidrule{2-6}
			& Mean Baseline  & 3 & 4 &  $.141_{.016}$ & $.701_{.016}$  \\
			\cmidrule{2-6}
            & Mean Baseline++ & 3 & 4 & $.622_{.024}$ & $.833_{.006}$ \\
            & Mean Clustered & $2.90+0.1$ & 60 & $.687_{.034}$ & $.846_{.002}$ \\
            & Median Clustered & $1.22+0.1$ & 60 & $.688_{.046}$ & $.860_{.004}$ \\
            \midrule
            \multirow{11}{*}{Yelp} & Real data & $\infty$ & - & $.874_{.012}$ & $.975_{.000}$ \\
			\cmidrule{2-6}
            & Mean Baseline \citep{amin2024private} & 10 & 2 &  $.136_{.014}$ & $.915_{.006}$  \\
            \cmidrule{2-6}
            & Mean Baseline++ & 10 & 2 & $.415_{.031}$ & $.899_{.014}$ \\
            & Mean Clustered & $9.90+0.1$ & 60 & $.449_{.021}$ & $.906_{.014}$ \\
            & Median Clustered & $2.21+0.1$ & 60 & $.460_{.019}$ & $.912_{.009}$ \\
			\cmidrule{2-6}
			& Mean Baseline & 3 & 2 &  $.136_{.012}$ & ${.880_{.014}}$  \\
			\cmidrule{2-6}
            & Mean Baseline++ & 3 & 2 & $.391_{.054}$ & $.907_{.007}$ \\
            & Mean Clustered & $2.90+0.1$ & 60 & $.436_{.032}$ & $.904_{.015}$ \\
            & Median Clustered & $1.38+0.1$ & 60 & $.451_{.038}$ & $.904_{.014}$ \\
            \midrule
            \multirow{11}{*}{NYT Topic} & Real data & $\infty$ & - & $.863_{.009}$ & $.919_{.001}$ \\
			\cmidrule{2-6}
            & Mean Baseline \citep{amin2024private} & 10 & 8 & $.151_{.013}$ & $.668_{.024}$  \\
            \cmidrule{2-6}
            & Mean Baseline++ & 10 & 8 & $.613_{.045}$ & $.776_{.008}$ \\
            & Mean Clustered & $9.90+0.1$ & 80 & $.716_{.038}$  & $.796_{.001}$  \\
            & Median Clustered & $5.04 + 0.1$ & 80 & $.681_{.043}$ & $.797_{.001}$ \\
			\cmidrule{2-6}
			& Mean Baseline & 3 & 8 &  $.155_{.018}$ & $.668_{.026}$  \\
			\cmidrule{2-6}
            & Mean Baseline++ & 3 & 8 & $.637_{.055}$ & ${.782_{.007}}$ \\
            & Mean Clustered & $2.90+0.1$ & 80 & $.665_{.017}$ & $.788_{.002}$ \\
            & Median Clustered & $1.72 + 0.1$ & 80 & $.659_{.053}$ & $.780_{.006}$ \\
            \bottomrule
		\end{tabular}
		}
	    \vspace{10pt}
        \caption{Performance of our methods compared to \emph{Mean Baseline} (the algorithm of \cite{amin2024private}). We report the mean and std of MAUVE against real data (5 seeds) and downstream accuracy of a BERT model trained on the synthetic data (3 seeds). Our improved baseline (\emph{Mean Baseline++}) shows sharp increases in MAUVE across all settings, as well as classification accuracy on AGNews and NYT Topic.
        On top of this stronger baseline, gains from clustering stack, and lead to consistent and direct improvements to MAUVE across all settings. Baselines that do not employ clustering are still batched with other examples of the same label; our datasets have 4, 2, and 8 labels respectively. For results employing clustering, we report the privacy cost of inference as well as the $\varepsilon=0.1$ cost of cluster rebalancing. \emph{Median Clustered} achieves better or comparable quality when compute-and-output-token-matched, while admitting a tight ex-post data-dependent DP analysis.}\label{table:results}
        \vspace{-15pt}
\end{table} 

\paragraph{Evaluation.} We evaluate all methods on the following two metrics. \textbf{(1) BERT Accuracy:} we train a BERT model on generated synthetic data and report its final accuracy on a held-out set consisting of real data. To compute BERT accuracy, we split the synthetic data into a synthetic train and validation set for model selection, and applying the best checkpoint on real data.
\textbf{(2) MAUVE Score:} this metric measures the distributional similarity between the real private data and synthetic data. Mauve score ranges from 0 to 1 and a score indicates better alignment with the original data distribution and, therefore, higher-quality synthetic data. We compute MAUVE with Gecko embeddings \cite{lee2024gecko}, using 1K samples from both sets.

\paragraph{Baselines.} We implement the method of \cite{amin2024private} as a baseline, which only differs algorithmically in batching, and that the aggregated sampling logits is obtained via the mean rather than the median. Other DP inference synthetic data approaches in the literature \citep{tang2024privacy, gao2025dataadaptive} focus on generating few-shot examples for prompting, and have not demonstrated the ability to generate enough data ($\geq$2K examples) to compute MAUVE or finetune BERT. \emph{Mean Baseline} is the setup described in \cite{amin2024private}, using an IT model and prompt, 1 example per context, and within-label batching. \emph{Mean Baseline++} uses 2 examples in-context and switches to a pre-trained checkpoint and prompt. Due to computational constraints, we tuned hyperparameters on AGNews and fixed them for the other datasets. For all experiments, we use a sampling temperature of 1.5. We use 64 parallel contexts for $\varepsilon=10$ experiments and 256 parallel contexts for $\varepsilon=3$ experiments.

\paragraph{Privacy budget.} We report results for two settings: $\eps=3$ and $\eps=10$. For mean aggregation, we set the approximate differential parameter $\delta = (\texttt{dataset\_size})^{-1.1}$. The privacy budget includes the total $\eps$ used for both clustering and generation. For mean aggregation, we report unconditional $\varepsilon$. For median aggregation, we match the number of tokens generated and batch size (thereby matching output quantity and compute requirements), and report the resultant data-dependent ex-post $\varepsilon$.

\paragraph{Clustering and batching.} We start with 1000 clusters produced from DBPedia, and perform DP rebalancing as described in Section \ref{subsec:clusters} to target 60 clusters for Yelp and AGNews, and 80 clusters for NYT Topics (10 clusters for each of the 8 labels). For implementation reasons, we subdivide each cluster into fixed-sized batches, rather than random-sized batches as in Eq.~\eqref{eq:batch}.

\subsection{Results}

Table \ref{table:results} summarizes our main results on all three datasets. We demonstrate that public-cluster-informed batching as a drop-in replacement for naive batching improves over the baseline -- that is, simply adjusting the input batching to the algorithm and changing no other algorithmic details leads to significant improvements in representativeness. The same is the case for switching to the pretrained model, demonstrating stacking improvement. Furthermore, we show that our newly introduced median aggregation algorithm can achieve quality comparable or surpassing that of the mean algorithm, while admitting a tight, ex-post data-dependent DP analysis. We also remark that the privacy guarantee we give for the median algorithm is \emph{maximum} over all batches. In Appendix B, we plot the distribution of per-batch privacy costs (many batches are <50\% of the stated guarantee). Designing algorithms to take advantage of this property is an interesting avenue for future work.

\section{Related work}
\paragraph{Differentially private synthetic data.} Prior work on generating synthetic data with differential privacy guarantees can be broadly categorized into three categories:

\begin{itemize}
    \item \textbf{(A) Training-based} methods finetune language models on private data using differentially private stochastic gradient descent  \citep{yue2022synthetic,yue-etal-2023-synthetic,mattern-etal-2022-differentially,carranza-etal-2024-synthetic,kurakin2023harnessing,wang2024knowledgesg}. After training, the model is used to generate synthetic data. More recent studies \citep{wu2024prompt,tan2025synthesizing,tran2024differentially} leverage the abundance of public data by first finetuning the model on public datasets before applying differentially private finetuning on the private data.

\item \textbf{(B) API-based} methods generate synthetic data using only model APIs \citep{xie2024differentially,pmlr-v235-yu24e,wu2024privacypreserving,lin2024differentially,lin2025differentially}. They query the LLM with private examples and ask it to select the closest matching samples from a non-private dataset. They iteratively refine the output to ensure that it is similar to the private data. %

\item \textbf{(C) Inference-based} methods leverage private prediction \citep{dwork2018privacy}, which ensures the privacy of model outputs (i.e., predictions). A widely used approach to achieve this is privacy amplification by subsampling and private aggregation \citep{nissim2007smooth}. When this methodology is applied to LLMs, the model generates the next token for each subset of private data, and the predictions are then privately aggregated to produce the final output \citep{hong2023dp,amin2024private,tang2024privacy,gao2024data}.

\end{itemize}

\paragraph{Differentially private clustering.} Early foundational work on differentially private clustering \citep{wang2015differentially,su2016differentially,feldman2009private,nissim2016locating} established strong theoretical bounds for private clustering. Subsequent works have improved the practical aspects, focusing on balancing utility, efficiency, and privacy. The most common approach \citep{balcan2017differentially,chaturvedi2020differentially,cohen2022scalable} is to project the data into a lower dimension to reduce the additive error while preserving the relative distance, and then try to efficiently find good centers.

\section{Conclusion}

We have proposed a novel differentially private inference method for generating private synthetic data. Our method uses a clustering algorithm to group the input data into batches of similar examples, and leverages the resulting data homogeneity to generate high-quality synthetic data at significantly lower privacy cost. 

Our methods has some limitations, which may be addressed in future work. In contrast to ordinary (non-private) inference, the computational cost of running DP inference is linear in the batch size. On the other hand, data generation is \emph{embarassingly parallel} across batches, whereas many serial steps are required for DP training. While clustering improves the overall privacy and quality of synthetic data, computing clusters creates computational overhead and spends a (small) amount of additional privacy budget. Finally, the underlying technique for aggregating the next-token prediction requires more than pure-inference access to the underlying LLM. The algorithm needs access to the logits, specifically, which may not be exposed by the APIs of all commercial LLMs.

\section*{Author contributions}\label{sec:contributions}

\begin{itemize}
    \item {\bf Sara B} is the main contributor. She implemented the method and experimented and optimized its many variations.
    \item {\bf Alex B} contributed the final version of the main experiments and maintained the infrastructure for evaluating results.
    \item {\bf Umar S} proposed the core theoretical framework, including the median algorithm, and led the theoretical analysis.
    \item {\bf Kareem A} supervised {\bf Sara B} and organized the project and final paper.
    
    \item {\bf Sara B}, {\bf Alex B}, {\bf Umar S} and {\bf Kareem A} wrote the paper. 
    \item {\bf Everyone} contributed to refining the paper, discussing the broader impact, and framing of the work.
\end{itemize}

\newpage

\bibliographystyle{unsrtnat}
\bibliography{custom}

\newpage

\appendix
\section{Proof of Theorem \ref{thm:main}}
\label{sec:proofs}

\begin{defn}[Left-median, median and right-median] Let $Z \subseteq \CZ$ be a set of logit vectors. Let $\leftmedian(Z)$, $\median(Z)$, $\rightmedian(Z) \in \CZ$ be logit vectors, where the component of each vector corresponding to token $x$ is defined in terms of the multiset $Z^{(x)} = \{z_x : \bz \in Z\} \subseteq \bbR$ as follows:
\begin{itemize}
    \item If $|Z^{(x)}|$ is even, and $a$ and $b$ are the middle values in $Z^{(x)}$ (when all of the values are sorted), then $\leftmedian(Z)_x = a$, $\median(Z)_x = (a + b)/2$ and $\rightmedian(Z)_x = c$.
    \item If $|Z^{(x)}|$ is odd, and $a, b$ and $c$ are the middle values in $Z^{(x)}$ (when all of the values are sorted), then $\leftmedian(Z)_x = a$, $\median(Z)_x = b$, $\rightmedian(Z)_x = c$.
\end{itemize}
Note that since $Z^{(x)}$ is a multiset, it may contain repeated values, and therefore for any token $x$ it can happen that any of the consecutive values above are equal.
\end{defn}

The quantities in Definition \ref{defn:privacycostfunction} below depend on $\tau > 0$, but we have dropped this dependence from the notation to reduce clutter.

\begin{defn}[Per-token privacy cost function] \label{defn:privacycostfunction} For any set of logit vectors $Z \subseteq \CZ$ and token $x \in \CX$ let
\begin{align*}
\alpha(Z, x) &= \exp((\baz_x - \baz^{\rightop}_x)/\tau) \cdot \frac{\sum_y \exp(\baz^{\leftop}_y /\tau)}{\sum_y \exp(\baz_y /\tau)}\\
\beta(Z, x) &= \exp((\baz_x - \baz^{\leftop}_x)/\tau) \cdot \frac{\sum_y \exp(\baz^{\rightop}_y /\tau)}{\sum_y \exp(\baz_y /\tau)}\\
\gamma(Z, x) &= \max\left\{\log \frac{1}{\alpha(Z, x)}, \log \beta(Z, x)\right\}
\end{align*}
where $\bbaz^{\leftop} = \leftmedian(Z)$, $\bbaz = \median(Z)$ and $\bbaz^{\rightop} = \rightmedian(Z)$.
\end{defn}

\begin{lem} \label{lem:medianbounds} Let $Z, Z' \subseteq \CZ$ be neighboring sets of logit vectors. For each token $x \in \CX$ we have
\[
\leftmedian(Z)_x \le \median(Z')_x \le \rightmedian(Z)_x
\]\end{lem}
\begin{proof} Adding or removing a value from a multiset either leaves the median of the multiset unchanged, or shifts the median to the next higher or next lower value.
\end{proof}

\begin{lem} \label{lem:oneiteration} In Algorithm \ref{alg:main}, suppose that batch $S_i$ is replaced by neighboring batch $S'_i$. For all $t \ge 1$, token $x \in \CX$ and token sequence $\bx \in \CX^{t-1}$
\[
\alpha(Z_{i,t}, x) \le \frac{\Pr[x_{i, t} = x ~|~ \bx_{i, t-1} = \bx]}{\Pr[x'_{i, t} = x ~|~ \bx'_{i, t-1} = \bx]} \le \beta(Z_{i,t}, x)
\]
where $\bx'_{i, t} = (x'_{i, 1}, \ldots, x'_{i, t})$ is the token sequence generated when processing batch $S'_i$.\end{lem}
\begin{proof} Let $\bbaz = \median(Z_{i,t})$, $Z' = \{\clip_c(\logits(\bs\bx)) : \bs \in S'_i\}$ and $\bbaz' = \median(Z')$. We have
\begin{align}
\Pr[x_{i, t} = x ~|~ \bx_{i, t-1} = \bx] &= \frac{\exp(\baz_x /\tau)}{\sum_y \exp(\baz_y /\tau)} \notag\\
&= \frac{\exp(\baz'_x /\tau)}{\sum_y \exp(\baz_y /\tau)} \cdot \exp((\baz_x - \baz'_x) /\tau) \notag\\
&= \frac{\exp(\baz'_x /\tau)}{\sum_y \exp(\baz'_y /\tau)} \cdot \exp((\baz_x - \baz'_x)/\tau) \cdot \frac{\sum_y \exp(\baz'_y /\tau)}{\sum_y \exp(\baz_y /\tau)} \notag\\
&= \Pr[x'_{i, t} = x ~|~ \bx'_{i, t-1} = \bx] \cdot \exp((\baz_x - \baz'_x)/\tau) \cdot \frac{\sum_y \exp(\baz'_y /\tau)}{\sum_y \exp(\baz_y /\tau)} \label{eq:one}
\end{align}
Continuing from above
\begin{align*}
\textrm{Eq.~\eqref{eq:one}} &\ge \Pr[x'_{i, t} = x ~|~ \bx'_{i, t-1} = \bx] \cdot \exp((\baz_x - \baz^{\rightop}_x)/\tau) \cdot \frac{\sum_y \exp(\baz^{\leftop}_y /\tau)}{\sum_y \exp(\baz_y /\tau)} & \because \textrm{Lemma \ref{lem:medianbounds}}\\
&= \Pr[x'_{i, t} = x ~|~ \bx'_{i, t-1} = \bx] \cdot \alpha(Z_{i,t}, x)
\end{align*}
and
\begin{align*}
\textrm{Eq.~\eqref{eq:one}} &\le \Pr[x'_{i, t} = x ~|~ \bx'_{i, t-1} = \bx] \cdot \exp((\baz_x - \baz^{\leftop}_x)/\tau) \cdot \frac{\sum_y \exp(\baz^{\rightop}_y /\tau)}{\sum_y \exp(\baz_y /\tau)} & \because \textrm{Lemma \ref{lem:medianbounds}}\\
&= \Pr[x'_{i, t} = x ~|~ \bx'_{i, t-1} = \bx] \cdot \beta(Z_{i,t}, x) & \qedhere
\end{align*}
\end{proof}

We are now ready to prove Theorem \ref{thm:main}.

\begin{proof}[Proof of Theorem \ref{thm:main}] Let $D, D' \in \CD$ be neighboring datasets. For each seed, we can condition on a fixed value for the random integer $r$ selected in Eq.~\eqref{eq:batch}, since it is chosen independently of the dataset. Since the $\batch()$ function assigns each seed to one batch, there exists a single batch that differs by one seed when Algorithm \ref{alg:main} is run on input dataset $D$ instead of $D'$. Let $S_i$ and $S'_i$ be these neighboring batches. Let $x_{i,1}, \ldots, x_{i, n}$ and $x'_{i, 1}, \ldots, x'_{i, t}$ be the sequences of tokens generated when processing $S_i$ and $S'_i$, respectively. For each $t \in [n]$ let $\bx_{i,t} = (x_{i, 1}, \ldots, x_{i, t})$ and $\bx'_{i, t} = (x'_{i, 1}, \ldots, x'_{i, n})$ denote the first $t$ tokens of $\bx_{i, n}$ and $\bx'_{i, n}$, respectively. Also fix a token sequence $\by_n = (y_1, \ldots, y_n) \in \CX^n$, and for each $t \in [n]$ let $\by_t = (y_1, \ldots, y_t)$ denote the first $t$ tokens of $\by_n$. We have
\[
\frac{\Pr[\bx_{i, n} = \by_n]}{\Pr[\bx'_{i, n} = \by_n]} = 
\frac{\Pr[x_{i, 1} = y_1]}{\Pr[x'_{i, 1} = y_1]} \cdot \frac{\Pr[x_{i, 2} = y_2] ~|~ \bx_{i, 1} = \by_1]}{\Pr[x'_{i, 2} = y_2 ~|~ \bx'_{i, 1} = \by_1]} \cdots \frac{\Pr[x_{i, n} = y_n] ~|~ \bx_{i, n-1} = \by_{n-1}]}{\Pr[x'_{i, n} = y_n ~|~ \bx'_{i, n-1} = \by_{n-1}]}
\]
Taking logarithm of both sides and applying Lemma \ref{lem:oneiteration} we have
\[
\sum_{t=1}^n \log \alpha(Z_{i,t}, x_{i,t}) \le \log \frac{\Pr[\bx_{i, n} = \by_n]}{\Pr[\bx'_{i, n} = \by_n]} \le \sum_{t=1}^n \log \beta(Z_{i,t}, x_{i, t})
\]
which proves the theorem.
\end{proof}

\newpage

\section{Experiments continued}
\label{sec:additional-experiment-details}

\subsection{Median mechanism privacy cost}

Unlike unconditional, worst-case $\varepsilon$ that admits \emph{uniform} per-token privacy costs, the median mechanism's privacy cost differs batch-by-batch and also position-by-position. Figures \ref{fig:agnews-per-batch-epsilon}

\begin{figure*}[h]
  \centering
  \includegraphics[width=0.45\textwidth, trim=0cm 0.2cm 0.0cm 0.2cm, clip]{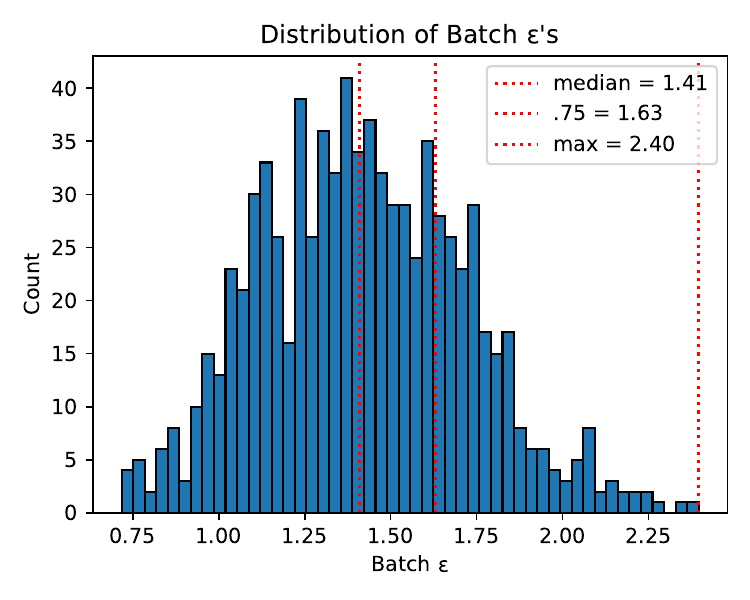} 
  \includegraphics[width=0.45\textwidth, trim=0cm 0.2cm 0cm 0.2cm, clip]{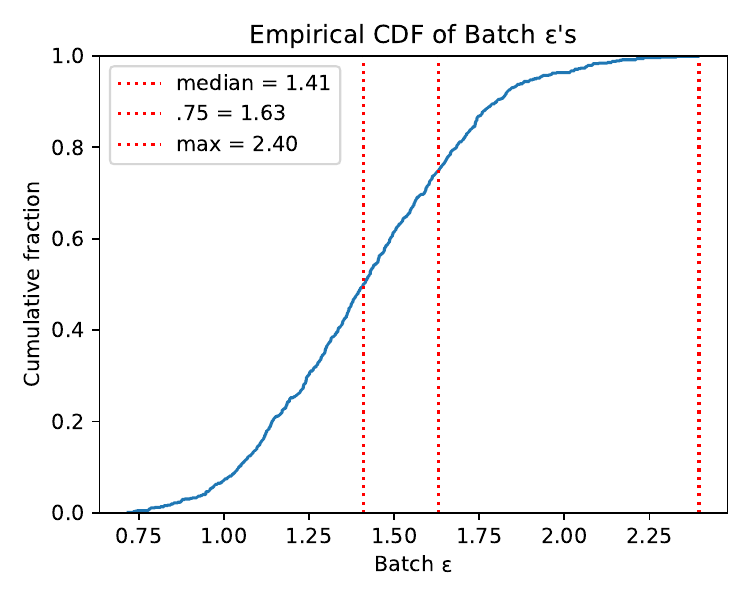} 
  \caption{We plot the distribution of \emph{per-batch $\varepsilon$ costs} of the median mechanism on AGNews. The maximum over all batches obtains $\varepsilon=2.40$, which is the privacy guarantee we report in Table \ref{table:results} via Theorem \ref{thm:main}. Most batches have substantially smaller privacy cost.}
  \label{fig:agnews-per-batch-epsilon}
\end{figure*}

\begin{figure*}[h]
  \centering
  \includegraphics[width=0.45\textwidth, trim=0cm 0.2cm 0.0cm 0.2cm, clip]{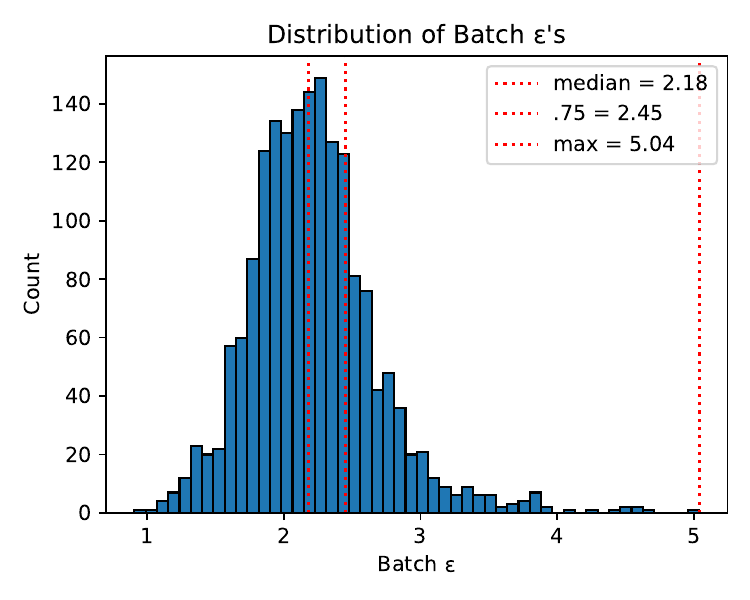} 
  \includegraphics[width=0.45\textwidth, trim=0cm 0.2cm 0cm 0.2cm, clip]{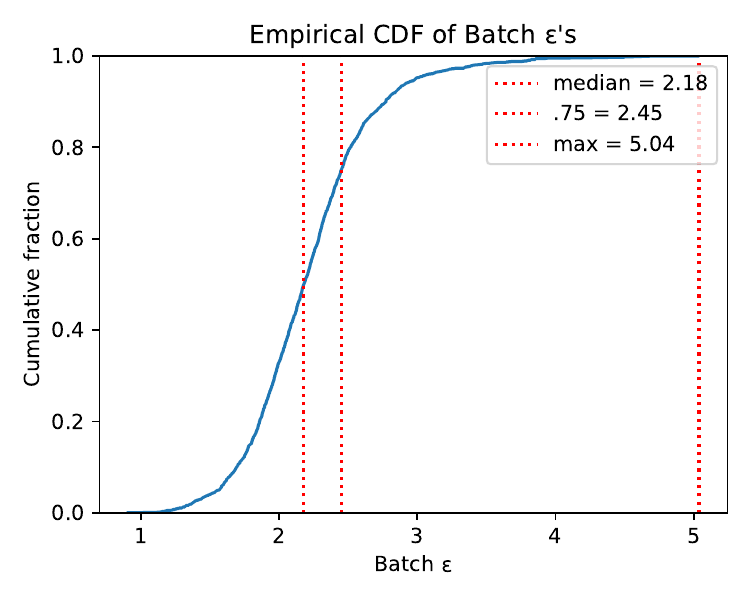} 
  \caption{We plot the distribution of \emph{per-batch $\varepsilon$ costs} of the median mechanism on NYT Topic. The maximum over all batches obtains $\varepsilon=5.04$, which is the privacy guarantee we report in Table \ref{table:results} via Theorem \ref{thm:main}. \emph{In this case, our accounting suffers particularly from the long right tail of per-batch privacy costs}.}
  \label{fig:nyt-topic-per-batch-epsilon}
\end{figure*}

\begin{figure*}[h]
  \centering
  \includegraphics[width=0.65\textwidth, trim=0cm 0.2cm 0.0cm 0.2cm, clip]{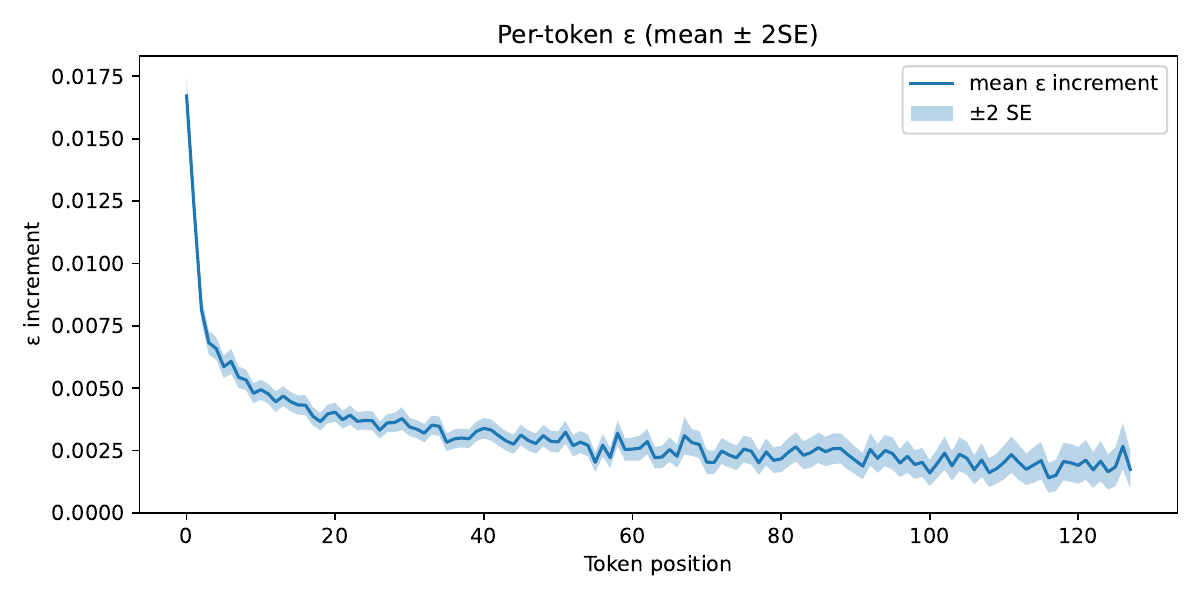} 
  \caption{We plot the average \emph{per-token $\varepsilon$ costs} of the median mechanism on AGNews ($\varepsilon=2.40$). Consensus builds throughout generation, decreasing the privacy cost.}
  \label{fig:agnews-per-token-epsilon}
\end{figure*}

\begin{figure*}[h]
  \centering
  \includegraphics[width=0.65\textwidth, trim=0cm 0.2cm 0.0cm 0.2cm, clip]{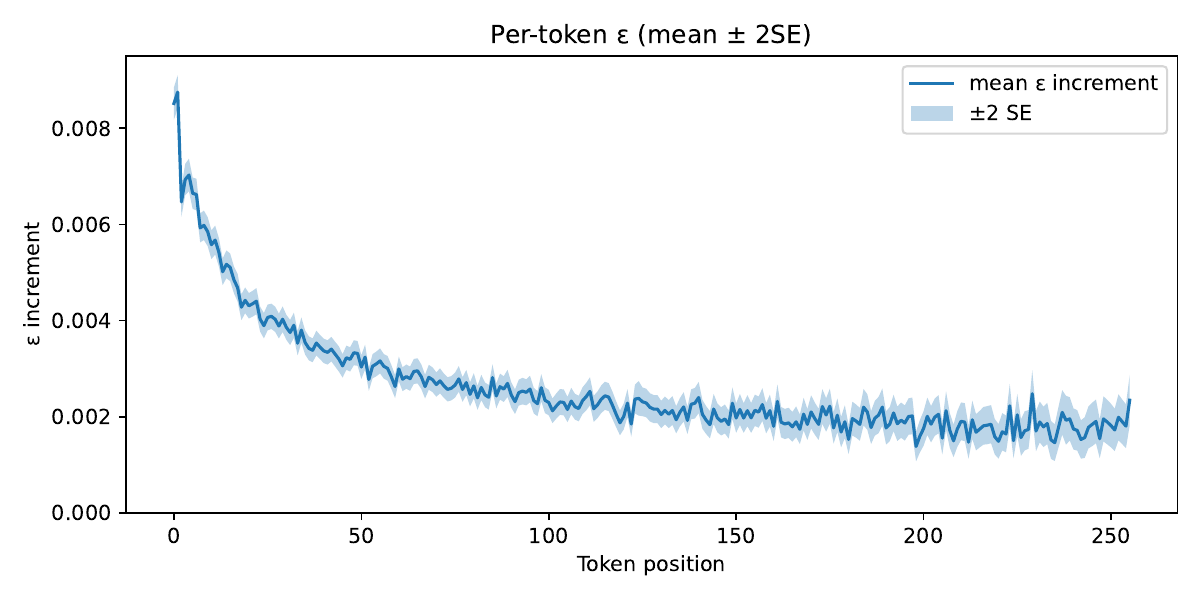} 
  \caption{We plot the average \emph{per-token $\varepsilon$ costs} of the median mechanism on Yelp ($\varepsilon=2.21$).}
  \label{fig:yelp-per-token-epsilon}
\end{figure*}

\newpage
\subsection{Evaluation hyperparameters}

\paragraph{MAUVE.} The absolute value of MAUVE scores can vary due to the precise implementation details, however the relative rankings it assigns to datasets is robust \citep{pillutla2021mauve}. We follow the original implementation\footnote{See \url{https://github.com/krishnap25/mauve/blob/main/src/mauve/compute_mauve.py}} closely, as well as report  all hyperparameters used: 768-dim Gecko embeddings \citep{lee2024gecko}, $n=1000$ texts per set, $n/10 = 100$ clusters (as recommended), $k$-means iteration limit of 500, 5 $k$-means initializations, PCA target explained variance of 0.9, MAUVE scaling factor of 5, and 32 MAUVE divergence curve discretization points.

\paragraph{BERT.} We compute BERT accuracy on a synthetic dataset by first partitioning it into a synthetic validation and synthetic train component, then running a hyperparameter sweep for BERT training, and finally selecting the checkpoint with best synthetic validation accuracy, and then finally reporting the accuracy of the selected checkpoint on real held-out data. 

The fraction of validation data is 0.1. We fix a batch size of 200, and train for roughly 500 steps by setting \texttt{epochs = math.ceil((batch\_size * steps)/train\_set\_size)}. We use Adam, and search over 5 learning rates [1e-6, 3e-6, 1e-5, 3e-5, 1e-4] $\times$ 2 settings for weight decay [0.0, 5e-4]. In each run, we employ early stopping: stopping after 4 epochs with no improvement in synthetic validation accuracy and returning the best checkpoint so far, in terms of synthetic validation accuracy.

\subsection{Method hyperparameters}

\begin{table}[H]
\small
    \centering
    \begin{tabular}{l l r r r r r }
        \toprule
        \multicolumn{1}{l}{\textbf{Setting}} & 
        \multicolumn{1}{l}{\textbf{Model}} &
        \multicolumn{1}{c}{\textbf{Examples $k$}} &
        \multicolumn{1}{c}{\textbf{Batch size}} &
        \multicolumn{1}{c}{\textbf{Output tokens}} &
        \multicolumn{1}{c}{\textbf{Temp.}} &
        \multicolumn{1}{c}{\textbf{Clip $c$}} \\   
        \midrule
        Baseline & IT & 1 & 64 & 1000 & 1.5 & 9 \\
        \midrule
        Baseline++ & PT & 2 & 64 & 1000 & 1.5 & 9 \\
        \hspace{0.5em} + \emph{non-private clustering} & PT & 2 & 64 & 1000 & 1.5 & 9 \\
        \bottomrule
    \end{tabular}
    \vspace{5pt}
    \caption{Hyperparameters for Yelp100k at $\varepsilon=\infty$ results presented in Table \ref{tab:clustering-demo}. $k$ refers to the number of examples in each context; $c$ is the clipping parameter in Equation \ref{eq:aggregate}.}
    \label{tab:clustering-demo-hyperparams_yelp}
\end{table}

\begin{table}[H]
\small
    \centering
    \begin{tabular}{l l r r r r r r}
        \toprule
        \multicolumn{1}{l}{\textbf{Setting}} &
        \multicolumn{1}{l}{$\boldsymbol{\varepsilon}$} & 
        \multicolumn{1}{l}{\textbf{Model}} &
        \multicolumn{1}{c}{\textbf{Examples $k$}} &
        \multicolumn{1}{c}{\textbf{Batch size}} &
        \multicolumn{1}{c}{\textbf{Output tokens}} &
        \multicolumn{1}{c}{\textbf{Temp.}} &
        \multicolumn{1}{c}{\textbf{Clip $c$}} \\   
        \midrule
        Mean Baseline & 10 & IT & 1 & 64 & 373/337/355 & 1.5 & 9 \\
        \midrule
        Mean Baseline++ & 10 & PT & 2 & 64 & 373/337/355 & 1.5 & 9 \\
        Mean Clustered & 9.9 & PT & 2 & 64 & 367/331/349 & 1.5 & 9 \\
        Median Clustered & - & PT & 2 & 64 & 367/331/349 & 1.5 & 6 \\
        \midrule
        Mean Baseline & 3 & IT & 1 & 256 & 733/642/686 & 1.5 & 9 \\
        \midrule
        Mean Baseline++ & 3 & PT & 2 & 256 & 733/642/686 & 1.5 & 9 \\
        Mean Clustered  & 2.9 & PT & 2 & 256 & 689/604/645 & 1.5 & 9 \\
        Median Clustered & - & PT & 2 & 256 & 689/604/645 & 1.5 & 6 \\
        \bottomrule
    \end{tabular}
    \vspace{5pt}
    \caption{Hyperparameters for $\varepsilon=3$ and $\varepsilon=10$ results presented in Table \ref{table:results}. The same hyperparameters are used across all datasets; except per-batch \emph{Output tokens} which depends on input dataset size to target the same $\varepsilon$; we report results for \emph{AGNews/Yelp/NYT Topic} respectively. $k$ refers to the number of examples in each context; $c$ is the clipping parameter in Equation. \ref{eq:aggregate}.}
    \label{tab:clustering-demo-hyperparams}
\end{table}

\subsection{Datasets and models}

\begin{table}[H]
\small
\centering
    \subfloat[Overview of datasets used. For synthesis targets, $n_\text{train}$ is 10\% smaller than reported elsewhere as we split off that amount to use for validation.]{
    \small
    \centering
    \scalebox{0.9}{
    \begin{tabular}{lrlll}
    \toprule
    Dataset & $n_\text{train}$ & Description & Usage & Source\\
    \midrule
    DBPedia & 560,000 & 14-category Wikipedia article topic & Public clusters & \citep{zhang2015character}\tablefootnote{\url{https://huggingface.co/datasets/fancyzhx/dbpedia_14}} \\
    \midrule
    AGNews & 108,000 & 4-way news topic classification & Synthesis target & \citep{zhang2015character}\tablefootnote{\url{https://huggingface.co/datasets/fancyzhx/ag_news}}\\
    Yelp Polarity & 504,000 & 2-way review sentiment classification & Synthesis target &\citep{zhang2015character}\tablefootnote{\url{https://huggingface.co/datasets/fancyzhx/yelp_polarity}}\\
    NYT Topics & 230,400& 8-way news topic classification & Synthesis target & \citep{singh_nyt_articles}\tablefootnote{\url{https://huggingface.co/datasets/dstefa/New_York_Times_Topics}} \\
    \bottomrule
    \end{tabular}
    }
    }
    \vspace{10pt}
    \\
    \subfloat[Overview of models used in experiments.]{
    \small
    \centering
    \scalebox{0.9}{
    \begin{tabular}{lll}
    \toprule
    Model  & Usage & Source\\
    \midrule
    Gecko & Generation; embeddings for clustering & \citep{lee2024gecko} \\
    Gemma 2 2B IT & \multirow{2}{*}{Generation; DP Inference}  & \citep{gemmateam2024gemma2}\tablefootnote{\url{https://huggingface.co/google/gemma-2-2b-it}}\\
    \multirow{1}{*}{Gemma 2 2B PT} & & \multirow{1}{*}{\citep{gemmateam2024gemma2}\tablefootnote{\url{https://huggingface.co/google/gemma-2-2b}}}\\
    \cmidrule{2-3}
    BERT-Base 12/768 110M & Evaluation; downstream finetuning & \citep{turc2019}\tablefootnote{\url{https://huggingface.co/google/bert_uncased_L-10_H-256_A-4}}\\
    Gecko & Evaluation; embeddings for MAUVE & \citep{lee2024gecko} \\
    \bottomrule
    \end{tabular}
    }
    }
    \vspace{10pt}
    \caption{Datasets and models used in our experiments. The Gecko embedding model is used for clustering, as well as for computing MAUVE. Gemma 2 2B IT/PT are used for synthetic data generation. We finetune BERT using synthetic data to evaluate how useful synthetic data is for improving accuracy on real data.}
\end{table}

\section{Prompts}
\label{sec:prompts}

PT and IT Gemma variants necessitate changes prompt changes. We use the same templates across all datasets. We stop generation when the model outputs its respective end token: "\texttt{\`{}\`{}\`{}}" for PT, "\texttt{<end\_of\_turn>}" for IT.

For clarity of exposition, we show the prompts when we use two examples per context, but the same template is generalizes to $k$ prompts per context (including $k=1$ used in our experiments). The field \texttt{label} is in natural language (e.g. \texttt{Positive} or \texttt{Negative} for Yelp Polarity); and recall that batches are constructed so that that all examples in a batch share a label.

\subsection{PT prompt template}

\begin{tcolorbox}[colback=lightgray,colframe=black,arc=4mm,boxrule=1pt,breakable]
\begin{verbatim}
```
{label}
{example1}
```

```
{label}
{example2}
```

```
\end{verbatim}
\end{tcolorbox}

\subsection{IT prompt template}

\begin{tcolorbox}[colback=lightgray,colframe=black,arc=4mm,boxrule=1pt,breakable]
\begin{verbatim}
<start_of_turn>user
Here are texts with Label: {label}.

Text: {example1}

Text: {example2}

Please give another one. No formatting or explanations.<end_of_turn>
<start_of_turn>model
Text: 
\end{verbatim}
\end{tcolorbox}

\end{document}